\documentclass[sigconf]{aamas_arxiv} 

\usepackage{balance} 

\usepackage[utf8]{inputenc} 
\usepackage[T1]{fontenc}    
\usepackage{hyperref}       
\usepackage{url}            
\usepackage{booktabs}       
\usepackage{amsfonts}       
\usepackage{nicefrac}       
\usepackage{microtype}      
\usepackage[ruled,vlined, linesnumbered]{algorithm2e}
\usepackage{amsmath}
\usepackage{amsthm}
\usepackage[colorinlistoftodos]{todonotes}
\usepackage{graphicx}
\usepackage{tikz}

\graphicspath{ {./images/} }

\newtheorem{theorem}{Theorem}
\newtheorem{lemma1}{Lemma}[section]
\newtheorem{lemma2}[lemma1]{Lemma}
\newtheorem{lemma3}[lemma1]{Lemma}

\setcopyright{ifaamas}
\acmConference[AAMAS '21]{Proc.\@ of the 20th International Conference on Autonomous Agents and Multiagent Systems (AAMAS 2021)}{May 3--7, 2021}{London, UK}{U.~Endriss, A.~Now\'{e}, F.~Dignum, A.~Lomuscio (eds.)}
\copyrightyear{2021}
\acmYear{2021}
\acmDOI{}
\acmPrice{}
\acmISBN{}

\title {The Effect of Q-function Reuse on the Total Regret of Tabular, Model-Free, Reinforcement Learning}

\author{Volodymyr Tkachuk}
\affiliation{
  \institution{University of Waterloo}
  \city{Waterloo, Ontario}}
\email{vtkachuk@uwaterloo.ca}

\author{Sriram Ganapathi Subramanian}
\affiliation{
  \institution{University of Waterloo}
  \city{Waterloo, Ontario}}
\email{s2ganapa@uwaterloo.ca}

\author{Matthew E. Taylor}
\affiliation{
  \department{University of Alberta \\ Alberta Machine Intelligence Institute (Amii)}
  \institution{Edmonton, Alberta}}
\email{matthew.e.taylor@ualberta.ca}

\begin{abstract}
  Some reinforcement learning methods suffer from high sample complexity causing them to not be practical in real-world situations. 
  $Q$-function reuse, a transfer learning method, is one way to reduce the sample complexity of learning, potentially improving usefulness of existing algorithms. 
  Prior work has shown the empirical effectiveness of $Q$-function reuse for various environments when applied to model-free algorithms. 
  To the best of our knowledge, there has been no theoretical work showing the regret of $Q$-function reuse when applied to the tabular, model-free setting.
  We aim to bridge the gap between theoretical and empirical work in $Q$-function reuse by providing some theoretical insights on the effectiveness of $Q$-function reuse when applied to the $Q$-learning with UCB-Hoeffding algorithm. 
  Our main contribution is showing that in a specific case if $Q$-function reuse is applied to the $Q$-learning with UCB-Hoeffding algorithm it has a regret that is independent of the state or action space. 
  We also provide empirical results supporting our theoretical findings.
\end{abstract}


\newcommand{\BibTeX}{\rm B\kern-.05em{\sc i\kern-.025em b}\kern-.08em\TeX}
\DeclareMathOperator{\argmaxH}{argmax}   
\newcommand*\mycirc[1]{%
   \begin{tikzpicture}
     \node[draw,circle,inner sep=0.5pt, scale=0.7] {#1};
   \end{tikzpicture}}

\begin{document}


\pagestyle{fancy}
\fancyhead{}


\maketitle 


\section{Introduction}

In reinforcement learning (RL), an agent interacts with an environment and tries to maximize its expected sum of rewards. 
Many algorithms, such as Q-learning with $\epsilon$-greedy exploration \cite{sutton2018reinforcement}, can suffer from poor sample complexity \cite{kearns2002near}. 
This is a problem in real-world situations where an agent may receive a limited amount of samples to learn an optimal policy. 
Such real-world environments serve as motivation to reduce the sample complexity of RL algorithms.

Transfer learning (TL) is a method used in RL as one way to reduce an agent’s training time \cite{taylor2009transfer}. 
The key idea is that an agent can learn a target task faster by transferring information from a previously learned source task, similar to how humans can learn algebra more quickly by transferring knowledge from previous tasks that require addition and multiplication. 
Although the concept of TL is intuitively appealing, its effectiveness has been mostly shown through empirical studies \cite{zhu2020transfer}.
As such, we aim to provide new theoretical results for one TL method in RL, $Q$-function reuse. 

$Q$-function reuse is the process of training an agent on a simple \emph{source} Markov decision process (MDP) $\mathcal{M}_S$ and then transferring its learned $Q$-function to a more complex, yet related \emph{target} MDP $\mathcal{M}_T$. The goal is to improve the sample complexity when compared to just training in $\mathcal{M}_T$ from the start. Sample complexity is loosely defined as how much data an agent must collect in order to learn a good policy \cite{kakade2003sample}. 
If the agent was trained in $\mathcal{M}_S$ until convergence to the optimal policy, transferring the $Q$-function from $\mathcal{M}_S$ to $\mathcal{M}_T$ can sometimes be thought of as a near-optimal $Q$-function initialization in $\mathcal{M}_T$, since $\mathcal{M}_S$ is related to $\mathcal{M}_T$.
Therefore, we propose that one method to study $Q$-function reuse is to study the effectiveness of near-optimal $Q$-function initialization. 

Since the effectiveness of $Q$-function reuse has been mostly shown in the model-free setting \cite{zhu2020transfer}, and it is easier to perform a theoretical analysis in the tabular domain, we choose to study the effects of $Q$-function reuse on a tabular, model-free algorithm that is provably efficient, $Q$-learning with UCB-Hoeffding \cite{jin2018q}.
In this work we study the setting where we are given the $Q$-function from some agent that has previously been trained on a simple MDP $\mathcal{M}_S$. 
We refer to this $Q$-function as the pre-trained $Q$-function from $\mathcal{M}_S$. 
We will answer the following question: 
\begin{center}
\fbox{\begin{minipage}{0.45\textwidth}
Will the total regret of the $Q$-learning with UCB-Hoeffding algorithm be lower in a complex (target) MDP $\mathcal{M}_T$, if it is initialized with a pre-trained $Q$-function from a related, but simpler (source) MDP $\mathcal{M}_S$? 
\end{minipage}}
\end{center}

For our analysis, we assume the $Q$-function initialization is optimal for all but one value.
Although this is a rather strong assumption, we believe it provides useful insights and a promising starting point for future work. 
In general, the $Q$-function initialization in a target MDP after $Q$-function reuse has some nearly optimal $Q$-values and some $Q$-values that are far from optimal.
Therefore, to address the general case of $Q$-function reuse, future work might include increasing the number of not optimal $Q$-values to more than one and relaxing the optimal $Q$-function initialization for the remaining $Q$-values to some notion of sub-optimality.
To the best of our knowledge, there has not been any theoretical work showing the total regret of $Q$-function reuse applied on a tabular, model-free algorithm.
We perform a regret analysis, showing that the $Q$-learning with the UCB-Hoeffding algorithm \cite{jin2018q}, along with our initialization assumptions achieves a total regret of only $O(\sqrt{H^2T\iota'})$ (independent of the state and action space), while regular $Q$-learning with UCB-Hoeffding suffers a regret of $O(\sqrt{H^4SAT\iota})$ \cite{jin2018q}. 
Empirical results are presented to support these theoretical claims.


\section{Preliminaries}

We borrow standard notations from \citet{jin2018q}, that we will include in this section for quick reference. 
To provide a fair comparison between our algorithm and the $Q$-learning with UCB-Hoeffding algorithm, proposed by \citet{jin2018q}, we maintain a similar problem setting. 
Extending this work to other settings (e.g., stochastic reward, terminating states,  discounting, etc.) is left for future work.
We begin by describing the Markov decision process, MDP$(\mathcal{S}, \mathcal{A}, H, \mathbb{P}, r)$. 
The set of states is $\mathcal{S}$, with $|\mathcal{S}| = S$. 
The set of actions is $\mathcal{A}$, with $|\mathcal{A}| = A$. The number of steps per episode (horizon) is $H$. 
The transition dynamics are given by $\mathbb{P}$, where $\mathbb{P}_h(\cdot|x, a)$ gives the next state distribution if action $a$ was taken in state $x$ at step $h \in [H]$.
The deterministic reward function is $r_h(x, a)$, which provides a reward in the range $[0, 1]$ for taking action $a$ in state $x$ at step $h$. 
The agent acts in this MDP for $K$ episodes. We let $T = K H$ denote the total number of steps the agent takes in the MDP. 

For each episode $k \in [K]$ an initial state $x_1^k$ is chosen randomly. 
At each step $h$ and episode $k$ the agent observes a state $x_h^k$, takes action $a_h^k$, receives reward $r_h(x_h^k, a_h^k)$, and then transitions to its next state drawn from the distribution $\mathbb{P}_h(\cdot|x_h^k, a_h^k)$. 
The transition dynamics and reward function were chosen to depend on the step $h$ for generality and to remain consistent with prior work \cite{jin2018q}. 
Note that if dependence on $h$ is not required then the transition dynamics and reward function can be set the same for all $h$ and all the results shown in this work will still hold.
The episode ends when $x_{H+1}^k$ is reached.

There is a separate policy $\pi_h$ for each step $\{\pi_h: \mathcal{S} \to \mathcal{A}\}_{h \in [H]}$. 
We use $V_h^{\pi}: \mathcal{S} \to \mathbb{R}$ to denote the value function at step $h$ under policy $\pi$. 
The expected sum of rewards under policy $\pi$, from $x_h=x$ until the end of the episode is given by $V_h^{\pi}(x)$. 
This is represented as:
$$V_h^{\pi}(x) := \mathbb{E}\left[ \sum_{h'=h}^H r_{h'}(x_{h'}, \pi_{h'}(x_{h'}))|x_h = x \right]$$

Similarly, we define the state-action value function as $Q_h^\pi: \mathcal{S} \times \mathcal{A} \to \mathbb{R}$. 
The expected sum of rewards under policy $\pi$, from state $x_h=x$ after taking action $a_h = a$ until the end of the episode is given by $Q_h^\pi (x, a)$. 
This is represented as: 
$$Q_h^{\pi}(x, a) := r_h(x, a) + \mathbb{E} \left[\sum_{h'= h+1}^H r_{h'} (x_{h'}, \pi_{h'}(x_{h'})) | x_h = x, a_h = a \right]$$

Since the state space, action space, and horizon are all finite, there always exists an optimal policy $\pi^*$ which gives the optimal value $V_h^*(x) := \text{sup}_\pi V_h^\pi(x)$, for all $x \in \mathcal{S}$ and $h \in [H]$ \cite{azar2017minimax}. 
To simplify notation we will denote $[\mathbb{P}_h V_{h+1}](x, a) := \mathbb{E}_{x'\sim\mathbb{P}(\cdot|x, a)} V_{h+1}(x')$. Using this notation we have the Bellman equations and Bellman optimality equations as follows:

\begin{equation}
\label{Bellman equations}
    \begin{cases}
      V_h^\pi(x) = Q_h^\pi(x, \pi_h(x))\\
      Q_h^\pi(x, a) = (r_h + \mathbb{P}_h V_{h+1}^{\pi})(x, a)\\
      V_{H+1}^\pi(x) = 0 \quad \forall x \in \mathcal{S}
    \end{cases} \\ 
    \begin{cases}
      V_h^*(x) = \text{max}_{a \in \mathcal{A}} Q_h^*(x, a)\\
      Q_h^*(x, a) := (r_h + \mathbb{P}_h V_{h+1}^*)(x, a)\\
      V_{H+1}^*(x) = 0 \quad \forall x \in \mathcal{S}
    \end{cases}
\end{equation}

We will use $\pi_h^k$ to denote the agent's policy at episode $k$. 
Finally, the performance metric of interest is the total regret, defined as:
$$\text{Regret}(K) = \sum_{k=1}^K[V_1^*(x_1^k) - V_1^{\pi_h^k}(x_1^k)]$$

\section{Results}

In this section we present our algorithm, $Q$-learning with UCB-Hoeffding and Max-Optimal Initialization, a modified version of Algorithm 1 from \citet{jin2018q}. We also introduce a theorem that shows the total regret of our algorithm is $O(\sqrt{H^2T\iota'})$. 

As a starting point to answering our question, presented in the introduction, we propose using an ideal $Q$-function initialization to model a possible pre-trained $Q$-function we might receive. 
We call this \textit{Max-Optimal Initialization} because it is the maximum number of assumptions that can be made before the $Q$-function initialization becomes the optimal $Q$-function for all states, actions, and steps. 
In words, we initialize the $Q$-function to the optimal $Q$-function for all $(x, a, h) \in \mathcal{S} \times \mathcal{A} \times [H]$ except for state $x1$, action $a1$ and step $h=1$, which is initialized to $H$ (as per the $Q$-learning with UCB-Hoeffding algorithm). 
Since updating the $Q$-function at any of the optimal states, actions, and steps could potentially make it sub-optimal, we make the additional assumption that the $Q$-function is only updated for $(x1, a1, h=1)$.
We only keep track of how many times $(x1, a1, h=1)$ was visited using a counter $N_1(x1, a1)$ initialized to $0$. 
Since the $Q$-function is not updated for any other $(x, a, h)$, there is no need to keep track of how many times any other $(x, a, h)$ is visited. 
In mathematical notation the initialization can be stated as follows:
\begin{equation}
\begin{aligned}
\label{Initialization equations}
    & Q_1(x1, a1) \leftarrow H \ \text{and} \ N_1(x1, a1) \leftarrow 0 \\
    & Q_h(x, a) \leftarrow Q^*_h(x, a) \ \text{for all} \ (x, a, h) \in \mathcal{S} \times \mathcal{A} \times [H] \backslash (x1, a1, h=1) 
\end{aligned}
\end{equation}

The Max-Optimal initialization can model the scenario where an agent is trained on some simple MDP $\mathcal{M}_0$ until convergence, but then a new action $a1$ is introduced at state $x1$ and step $h=1$ (MDP $\mathcal{M}_F$). 
In $\mathcal{M}_F$ the agent is essentially initialized with an optimal $Q$-function for all $(x, a, h)$, except $(x1, a1, h=1)$.
We would like to highlight that in this scenario $\mathcal{M}_0$ is considered related to MDP $\mathcal{M}_F$ since it has the same transition dynamics and reward function for all but one state, action, and step. $\mathcal{M}_0$ is also considered simpler than $\mathcal{M}_F$ since it is exactly $\mathcal{M}_F$, except without $(x1, a1, h=1)$. 
We present Algorithm \ref{algorithm 1}, $Q$-learning with UCB-Hoeffding and Max-Optimal Initialization, which combines the $Q$-learning with UCB-Hoeffding algorithm with the Max-Optimal Initialization assumptions. 
The modifications we made to the $Q$-learning with UCB-Hoeffding algorithm are shown in blue in Algorithm \ref{algorithm 1}.

We now describe the steps performed in our algorithm.
The $Q$-function and step counter $N_1(x1, a1)$ are initialized using Max-Optimal initialization.
For each episode $k \in [K]$ the agent starts in a random state $x_1^k$.
Then, for each step $h \in [H]$ and state $x_h^k \in \mathcal{S}$ the agent selects the action $a \in \mathcal{A}$ that maximizes its current estimate of $Q_h(x, a)$. 
The next state $x_{h+1}^k \in \mathcal{S}$ is sampled from $\mathbb{P}_h(\cdot|x, a)$. 
If the current state, action and step $(x_h^k, a_h^k, h)$ is $(x1, a1, h=1)$, the agent updates its $Q$-function using the following rule:
\begin{equation}
\label{update equation}
 Q_h(x_h^k, a_h^k) \leftarrow (1-\alpha_t) Q_h(x_h^k, a_h^k) + \alpha_t[r_h(x_h^k, a_h^k) + V_{h+1}(x_{h+1}^k) + b_t]
\end{equation}
where $t$ is the step counter for how many times the agent has visited $(x1, a1)$ at step $h=1$, $b_t$ is the confidence bonus indicating the agents confidence in its $Q$-value at $(x1, a1, h=1)$, and the learning rate $\alpha$ is $\alpha_t := \frac{H+1}{H+t}$.
This choice of learning rate $\alpha_t$ is crucial to obtain a total regret that is not exponential in $H$ \cite{jin2018q}.

We removed the update $V_h(x_h) \leftarrow \min\{H, \max_{a' \in \mathcal{A}}Q_h(x_h, a')\}$ since the $Q$-function in our algorithm is only updated for $(x1, a1, h=1)$, which requires knowledge of $V_{h+1}^k(x_{h+1}^k)$. 
But $V_{h+1}^k(x_{h+1}^k) = V_{h+1}^*(x_{h+1}^k)$ for $h \geq 1$ due to the Max-Optimal initialization, meaning we never need to update $V_{h+1}^k(x_{h+1}^k)$.

We present the following theorem for the $Q$-learning with UCB-Hoeffding and Max-Optimal Initialization algorithm:
\begin{theorem} [Hoeffding Max-Optimal]
\label{Hoeffding MO} 
There exists an absolute constant $c > 0$ such that, for any $p \in (0, 1)$, if we choose $b_t = c\sqrt{H^3\iota'/t}$, then with probability $1 - p$ the total regret of $Q$-learning with UCB-Hoeffding and Max-Optimal Initialization (Algorithm 1) is at most $O(\sqrt{H^2T\iota'})$, where $\iota' := log(K/p)$.
\end{theorem}
Note that we reserve $\iota$ for when we refer to the $Q$-learning with UCB-Hoeffding algorithm, where $\iota := log(SAT/p)$. 
In our algorithm this term is reduced to $\iota' := log(K/p)$ due to our added assumptions.


\begin{algorithm}[ht]
\SetAlgoLined
 \textcolor{blue}{Initialize $Q_1(x1, a1) \leftarrow H$ and $N_1(x1, a1) \leftarrow 0$} \\ 
 \textcolor{blue}{Initialize $Q_h(x, a) \leftarrow Q^*_h(x, a)$ for all $(x, a, h) \in \mathcal{S} \times \mathcal{A} \times [H] \backslash (x1, a1, h=1)$}\\ 
 \For{episode $k = 1,...,K$}{
  receive $x_1^k$\\
  \For{step $h = 1,...,H$}{
   Take action $a_h^k \leftarrow \argmaxH_{a'} {Q_h(x_h^k, a')}$, and observe $x_{h+1}^k$\\
   \If{\textcolor{blue}{$x_h^k = x1$ and $a_h^k = a1$ and $h = 1$}}{
    $t = N_h(x_h^k, a_h^k) \leftarrow N_h(x_h^k, a_h^k) + 1$; $b_t \leftarrow c\sqrt{H^3\iota'/t}$\\
    $Q_h(x_h^k, a_h^k) \leftarrow (1 - \alpha_t)Q_h(x_h^k, a_h^k) + \alpha_t[r_h(x_h^k, a_h^k) + V_{h+1}(x_{h+1}^k) + b_t]$\\
   }
  }
 }
 \caption{$Q$-learning with UCB-Hoeffding and Max-Optimal Initialization}
\label{algorithm 1}
\end{algorithm}



\section{Proof for Q-learning with UCB-Hoeffing and Max-Optimal Initialization}

We provide a full proof of Theorem \ref{Hoeffding MO}, following similar steps and notation to that mentioned in \citet{jin2018q}. We first introduce some notation for convenience.

We denote by $\mathbb{I}[A]$ as the indicator function for an event $A$. 
Recall that $[\mathbb{P}_h V_{h+1}](x, a) := \mathbb{E}_{x' \sim \mathbb{P}_h(\cdot | x, a)} V_{h+1}(x')$. We now introduce its empirical counterpart, $[\hat{\mathbb{P}}_h^k V_{h+1}](x, a) := V_{h+1}(x_{h+1}^k)$, which is defined only for $(x, a) = (x_h^k, a_h^k)$.
Recalling that $\alpha_t = \frac{H + 1}{H + t}$, we introduce the following:
\begin{equation}
\label{alpha notation}
    \alpha_t^0 = \prod_{j=1}^t (1-\alpha_j), \qquad \alpha_t^i = \alpha_i\prod_{j = i+1}^t(1-\alpha_j)
\end{equation}
Recall that $\sum_{j}^{t<j}(\cdot) = 0$ and $\prod_{j}^{t<j}(\cdot) = 1$.
Therefore, the following properties hold: 
$$\sum_{i=1}^t \alpha_t^i = 1 \ \text{and} \ \ \alpha_t^0 = 0 \ \text{for} \ t \geq 1, \  \sum_{i=1}^t \alpha_t^i = 0 \ \text{and} \ \alpha_t^0 = 1 \ \text{for} \ t = 0$$

The motivation for introducing this notation is to simplify the recursive $Q$-function update formula (as seen in equation (\ref{update equation})).
From equation (\ref{update equation}) and equation (\ref{alpha notation}) we have:
\begin{equation}
\label{Compact Q-update}
    Q_h^k(x, a) = \alpha_t^0 H + \sum_{i=1}^t \alpha_t^i[r_h(x,a) + V_{h+1}^*(x_{h+1}^{k_i}) + b_i]
\end{equation}
which only applies for $(x1, a1, h=1), \ \forall k \in [K]$, as discussed in equation (\ref{update equation}). 

\subsection*{Proof Details}
We now introduce some Lemmas that will help us in the proof of Theorem \ref{Hoeffding MO}.
For completeness we repeat Lemme 4.1 exactly as stated in \citet{jin2018q}. 

\begin{lemma1}
\label{alpha properties}
The following properties hold for $\alpha_t^i$:
 \renewcommand{\theenumi}{\alph{enumi}}
 \begin{enumerate}
   \item $\frac{1}{\sqrt{t}} \leq \sum_{i=1}^t \frac{\alpha_t^i}{\sqrt{i}} \leq \frac{2}{\sqrt{t}}$ for every $ t \geq 1$
   \item $\max_{i \in [t]} \alpha_t^i \leq \frac{2H}{t}$ and $\sum_{i=1}^t(\alpha_t^i)^2 \leq \frac{2H}{t}$ for every $t \geq 1$
   \item $\sum_{t=1}^\infty \alpha_t^i = 1 + \frac{1}{H}$ for every $i \geq 1$
 \end{enumerate}
\end{lemma1}

\begin{proof}
    See proof of Lemma 4.1 in \citet{jin2018q}
\end{proof}

We now present modified versions of Lemma 4.2 and Lemma 4.3 from \citet{jin2018q}, which we will use in our proof of Theorem \ref{Hoeffding MO}.

\begin{lemma2} [Difference in Q] 
\label{Diff in Q}
For $(x1, a1, h=1) \in \mathcal{S} \times \mathcal{A} \times [H]$ and episode $k \in [K]$, let $t = N_1^k(x1, a1)$ and suppose $(x1, a1)$ was previously taken at step h=1 of episode $k_1,...,k_t < k$. Then:
\begin{align*}
(Q_1^k - Q_1^*)(x1, a1) = 
&\alpha_t^0(H - Q_1^*(x1, a1))  \\
&+ \sum_{i=1}^t \alpha_t^i \left[ [(\hat{\mathbb{P}}_1^{k_i} - \mathbb{P}_1)V_{2}^*](x1, a1) + b_i \right]
\end{align*}
\end{lemma2}

\begin{proof} [Proof of Lemma \ref{Diff in Q}]
From the Bellman optimality equation we have $Q_h^*(x, a) = (r_h + \mathbb{P}_h V_{h+1}^*)(x, a)$.
Recalling that $[\hat{\mathbb{P}}_h^k V_{h+1}](x, a) := V_{h+1}(x_{h+1}^k)$, and the fact that $\sum_{i=0}^t \alpha_t^i = 1$, we have:
\begin{align*}
    Q_h^*(x, a) = \ & r_h(x, a) + [\mathbb{P}_h V_{h+1}^*](x, a) \\ 
    = \ & \alpha_t^0 Q_h^*(x, a) + \sum_{i=1}^t \alpha_t^i \left[r_h(x, a) + [\mathbb{P}_h V_{h+1}^*](x, a)\right] \\ 
    = \ &\alpha_t^0 Q_h^*(x, a) \\ 
    &+ \sum_{i=1}^t \alpha_t^i \left[r_h(x, a) + [(\mathbb{P}_h - \hat{\mathbb{P}}_h^{k_i})V_{h+1}^*](x, a) + V_{h+1}^*(x_{h+1}^{k_i})\right]  
\end{align*}
which is true for all $(x, a, h) \in \mathcal{S} \times \mathcal{A} \times [H]$. Subtracting the above equation from formula (\ref{Compact Q-update}) for $(x1, a1, h=1)$, we obtain Lemma \ref{Diff in Q}:
 
\begin{align*}
    (Q&_1^k - Q_1^*)(x1, a1) \\
    = \ & \alpha_t^0(H - Q_1^*(x1, a1)) \\ 
    &+ \sum_{i=1}^t \alpha_t^i \left[(V_{2}^{k_i} - V_{2}^*)(x_{2}^{k_i}) + [(\hat{\mathbb{P}}_1^{k_i} - \mathbb{P}_1)V_{2}^*](x1, a1) + b_i\right]\\ 
    \stackrel{\mycirc{1}}{=} \ & \alpha_t^0(H - Q_1^*(x1, a1)) \\
    &+ \sum_{i=1}^t \alpha_t^i \left[ [(\hat{\mathbb{P}}_1^{k_i} - \mathbb{P}_1)V_{2}^*](x1, a1) + b_i \right]
\end{align*}
Where \textcircled{1} holds because $V_h^k(x_h^k) = V_h^*(x_h^k)$ for $h \geq 2, \forall k \in [K]$
\end{proof}

Next, we present a modified version of Lemma 4.3 from \citet{jin2018q}, for $(x1, a1, h=1)$, which shows that $Q_1^k$ is an upper bound on $Q_1^*$ with high probability.
 
\begin{lemma3} [bound on $(Q^k_1 - Q^*_1)(x1, a1)$]
\label{bound on Q}
There exists an absolute constant $c > 0$ such that, for any $p \in (0, 1)$, letting $b_t = c\sqrt{H^3\iota'/t}$, we have $\beta_t = 2\sum_{i=1}^t \alpha_t^i b_i \leq 4c\sqrt{H^3\iota'/t}$, where $\iota' = log(K/p)$, and with probability at least $1-p$, the following holds for $(x1, a1, h=1), \ \forall k \in [K]$:
$$0 \leq (Q_1^k - Q_1^*)(x1, a1) \leq \alpha_t^0 H + \beta_t$$
where $t = N_1^k(x1, a1)$ and $k_1,...,k_t < k$ are the episodes where $(x1, a1)$ was taken at step $h=1$.
\end{lemma3}

\begin{proof} [Proof of Lemma \ref{bound on Q}] 
For $(x1, a1, h=1)$, let us denote $k_0=0$ and denote
$$k_i = \min\left(\{k \in [K] \ | \ k \geq k_{i-1} \wedge (x_1^k, a_1^k) = (x1, a1)\} \cup \{K + 1\}\right)$$
In words, this means that $k_i$ is the episode when $(x1, a1)$ was taken at step $h=1$ for the $i$th time, and $k_i$ equals $K + 1$ if $(x1, a1)$ was taken for fewer than $i$ times. 
The random variable $k_i$ can be thought of as a stopping time. 
If we let $\mathcal{F}_i$ (filtration) be the $\sigma$-algebra generated by all the random variables until episode $k_i$, and step $h=1$. 
Then, $\big(\mathbb{I}[k_i \leq K] \cdot [(\hat{\mathbb{P}}_1^{k_i} - \mathbb{P}_1) V_2^*](x1, a1)\big)_{i=1}^\tau$ is a Martingale difference sequence w.r.t.\ 
the filtration $\{\mathcal{F}_i\}_{i\geq0}$.
In words, the filtration $\{\mathcal{F}_i\}_{i\geq 0}$ can be thought of as a sequence of increasing information about state $x1$, action $a1$ at step $h=1$, where the filtration satisfies $\mathcal{F}_1 \subseteq \mathcal{F}_2 ... \subseteq \mathcal{F}_i$ and the three properties that define a $\sigma$-algebra. 
By Azuma-Hoeffding and a union bound, we have with probability at least $1 - p$:
\begin{align}
\label{A-H and Union bound}
    \forall \tau \in [K] 
    &: \ \left| \sum_{i=1}^\tau \alpha_\tau^i \cdot \mathbb{I}[k_i \leq K] \cdot [(\hat{\mathbb{P}}_1^{k_i} - \mathbb{P}_1)V_{2}^*](x1, a1) \right| \nonumber \\
    &\leq \frac{cH}{2}\sqrt{\sum_{i=1}^\tau (\alpha_\tau^i)^2 \cdot \iota'} \leq c\sqrt{\frac{H^3\iota'}{\tau}}
\end{align}
for some absolute constant c. Because inequality (\ref{A-H and Union bound}) holds for all fixed $\tau \in [K]$ uniformly, it also holds for $\tau = t = N_1^k(x1, a1) \leq K$, which is a random variable, where $k \in [K]$. 
Also note $\mathbb{I}[k_i \leq K] = 1$ for all $i \leq N_1^k(x1, a1)$. 
We now have:   

\begin{equation}
\label{bonus bound}
     \left| \sum_{i=1}^t \alpha_\tau^i \cdot [(\hat{\mathbb{P}}_1^{k_i} - \mathbb{P}_1)V_{2}^*](x1, a1) \right| \leq c\sqrt{\frac{H^3\iota'}{\tau}} \ \ \text{where} \ \ t=N_1^k(x1, a1)
\end{equation}
If we choose $b_t = c\sqrt{H^3\iota'/t}$ for the same constant $c$ as in inequality (\ref{bonus bound}), we have $\beta_t/2 = \sum_{i=1}^t \alpha_t^ib_i \in [c\sqrt{H^3\iota'/t}, 2c\sqrt{H^3\iota'/t}]$ according to Lemma \ref{alpha properties}.a. Then the right-hand side of Lemma \ref{bound on Q} follows immediately from Lemma \ref{Diff in Q} and inequality (\ref{bonus bound}). The left-hand side also follows from Lemma \ref{Diff in Q}.
\end{proof}

\noindent We are now ready to prove Theorem \ref{Hoeffding MO}. 

\begin{proof} [Proof of Theorem \ref{Hoeffding MO}]
We follow a similar procedure to \citet{jin2018q}, except we do not have to decompose the regret into a recursive form in $h$ since we are interested in a fixed $h=1$. 
We denote $\delta_1^k := (V_1^k - V_1^{\pi_1^k})(x1)$.

By Lemma \ref{bound on Q}, we have that with at least $1 - p$ probability, $Q_1^k(x1, a1) \geq Q_1^*(x1, a1)$ and thus $V_1^k(x1) \geq V_1^*(x1)$. We also know that $Q_1^k(x, a) =  Q_1^*(x, a), \ \forall (x, a) \in \mathcal{S} \times \mathcal{A} \backslash (x1, a1)$ and thus $V_1^k(x) = V_1^*(x), \ \forall x\in \mathcal{S} \backslash x1$. The regret can be upper bounded:
\begin{align}
 \text{Regret}(K) 
 &= \sum_{k=1}^K(V_1^* - V_1^{\pi_1^k})(x_1^k) \nonumber 
 \leq \sum_{k=1}^K(V_1^k - V_1^{\pi_1^k})(x_1^k) \nonumber \\ 
 &\stackrel{\mycirc{1}}{\leq} \sum_{k=1}^K(V_1^k - V_1^{\pi_1^k})(x1)
 = \sum_{k=1}^K \delta_1^k \label{regret bound with delta}
\end{align}
where inequality \textcircled{1} holds because $\sum_{k=1}^K(V_1^k - V_1^{\pi_1^k})(x_1^k) = 0, \  \forall x \in \mathcal{S} \backslash x1$.

For any fixed $k, \ \in [K]$, let $t = N_1^k(x1, a1)$ and suppose $(x1, a1)$ was previously taken at step $h=1$ of episode $k_1,...,k_t < k$, Then we have:
\begin{align}
\label{delta bound}
    \delta_1^k &= (V_1^k - V_1^{\pi_1^k})(x1)
    \stackrel{\mycirc{1}}{=} (Q_1^k - Q_1^{\pi_1^k})(x1, a_1^k) \nonumber \\
    &= (Q_1^k - Q_1^*)(x1, a_1^k) + (Q_1^* - Q_1^{\pi_1^k})(x1, a_1^k) \nonumber \\
    &\stackrel{\mycirc{2}}{=} (Q_1^k - Q_1^*)(x1, a1) + (Q_1^* - Q_1^{\pi_1^k})(x1, a1) \nonumber \\
    &\stackrel{\mycirc{3}}{\leq} \alpha_t^0H + \beta_t + [\mathbb{P}_1(V_{2}^* - V_{2}^{\pi_1^k})](x1, a1) \nonumber \\
    &\stackrel{\mycirc{4}}{=} \alpha_t^0H + \beta_t
\end{align}
where $\beta = 2\sum \alpha_t^i b_i \leq O(1) \sqrt{H^3\iota'/t}$.
Equality \textcircled{1} holds because $V_1^k(x1) = \max_{a' \in \mathbb{A}} Q_1^k(x1, a') = Q_1^k(x1, a_1^k)$. 
Equality \textcircled{2} holds because $Q_1^k(x1, a_1^k) = Q_1^*(x1, a_1^k), \ \forall a_1^k \in \mathcal{A} \backslash a1$ and $ Q_1^*(x1, a_1^k) = Q_1^{\pi_h^k}(x1, a_1^k), \ \forall a_1^k \in \mathcal{A} \backslash a1$. 
Inequality \textcircled{3} holds with at least $1-p$ probability by Lemma \ref{bound on Q} and the Bellman equations (\ref{Bellman equations}). 
Finally, equality \textcircled{4} holds since $[\mathbb{P}_h(V_{h+1}^{\pi_h^k} - V_{h+1}^*)](x1, a1) = 0$ for $\ h \geq 1, \forall k \in [K]$.
We now compute the summation $\sum_{k=1}^K \delta_1^k$. Denoting $n_1^k = N_1^k(x1, a1)$, we have:
\begin{equation}
\label{H bound}
    \sum_{k=1}^K \alpha_{n_1^k}^0 H 
    = \sum_{k=1}^K H \cdot \mathbb{I}[n_1^k = 0] 
    \leq H
\end{equation}
We also have:
\begin{equation}
\label{beta bound}
    \sum_{k=1}^K \beta_{n_1^k} 
    \leq O(1) \cdot \sum_{k=1}^K \sqrt{\frac{H^3\iota'}{n_1^k}} 
    = O(1) \sum_{n=1}^{N_1^K(x1,a1)} \sqrt{\frac{H^3\iota'}{n}} 
    \stackrel{\mycirc{1}}{\leq} O(\sqrt{H^2T\iota'})
\end{equation}
where inequality \textcircled{1} is true because the left-hand side of \textcircled{1} is maximized when $N_1^k(x1, a1) = K$, and $\sum_{n=1}^K \sqrt{\frac{1}{n}}$ can be bounded by $O(\sqrt{K})$.
Taking the sum from $k=1$ to $K$ for (\ref{delta bound}) and plugging in (\ref{H bound}) and (\ref{beta bound})  we have:
$$\sum_{k=1}^K \delta_1^k \leq O(H + \sqrt{H^2T\iota'}) \stackrel{\mycirc{1}}{=} O(\sqrt{H^2T\iota'})$$
where \textcircled{1} is true since $\sqrt{H^2T\iota'} \geq H$.

Recall that Lemma \ref{bound on Q} was applied twice in this proof (once in equation (\ref{regret bound with delta}) and once in \textcircled{3} in equation (\ref{delta bound})). 
Note that $(1-p)^2 = 1-2p+p^2 \geq 1-2p$, where the last inequality is applied because the $p^2$ term is small when compared to $2p$.
In summary, we have that $\sum_{k=1}^K \delta_1^k \leq O(\sqrt{H^2T\iota'})$ holds with probability at least $1-2p$. 
Note the term $p$ cannot be greater than $1/2$ to ensure the probability $1-2p$ is non-negative. 
This can be achieved by re-scaling $p \to p/2$ to reduce it's range from $[0, 1]$ to $[0, 1/2]$. 
As such, re-scaling $p$ to $p/2$ finishes the proof.
\end{proof}


\section{Experimental Setup}

\begin{figure}[t]
    \centering
    \includegraphics[width=0.25\textwidth]{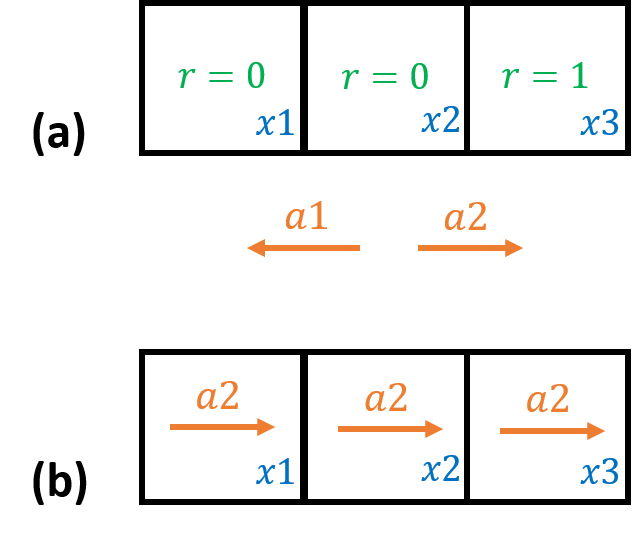}
    \caption{The environment the agent was trained in can be seen in (a). For the purpose of simplifying the image we use the notation $r(x1) = 0$, $r(x2)=0$, and $r(x3)=1$ to represent the reward the agent receives for transitioning into states $x1, x2, x3$ respectively. The optimal policy is shown in (b).}
    \label{fig:environment}
\end{figure}
In the previous section we proved the theoretical total regret of the $Q$-learning with Max-Optimal Initialization algorithm is bounded by $O(\sqrt{H^2T\iota'})$, while the total regret of the $Q$-learning with UCB-Hoeffding algorithm is bounded by $O(\sqrt{H^4SAT\iota})$ \cite{jin2018q}. 
This result implies that our statement of interest holds true theoretically. 
In this section our goal is to empirically show that results are consistent with our theoretical findings.
We do this by choosing a simple tabular environment to compare both algorithms. 
The simple environment will allow for easier interpretation of results. 
We expect the empirical results to hold for larger and more complex environments.

The environment used is a 1-dimensional gridworld, with $3$ states, and $2$ actions (see Figure \ref{fig:environment} (a)). 
The actions are left and right, and when the agent takes an action that causes it to hit a wall it remains in the same state. 
All transitions and rewards are deterministic. 
The agent receives a reward of $1$ if its next state is the right most state and a reward of $0$ otherwise. 
The agent interacts with the environment for exactly $3$ steps each episode. 
The optimal policy is to go right ($a2$) in all states for all steps (See Figure \ref{fig:environment} (b)).
Formally this setting can be represented by a MDP$(\mathcal{S}, \mathcal{A}, H, \mathbb{P}, r)$, where there are $3$ states $\mathcal{S} = \{x1, x2, x3\}$, $2$ actions $\mathcal{A} = \{a1, a2\}$, the horizon is $H = 3$, and the transition dynamics and reward function are as follows for all $h \in [H] = \{1, 2, 3\}$:
\begin{equation*}
\label{transition and reward functions}
    \begin{cases}
        \mathbb{P}_h(x1|x1, a1) = 1 \\ 
        \mathbb{P}_h(x2|x1, a2) = 1 \\ 
        \mathbb{P}_h(x1|x2, a1) = 1 \\
        \mathbb{P}_h(x3|x2, a2) = 1 \\
        \mathbb{P}_h(x2|x3, a1) = 1 \\
        \mathbb{P}_h(x3|x3, a2) = 1 \\ 
    \end{cases} \quad \text{and} \quad 
    \begin{cases}
        r_h(x1, a1) = 0\\
        r_h(x1, a2) = 0\\
        r_h(x2, a1) = 0\\
        r_h(x2, a2) = 1\\
        r_h(x3, a1) = 0\\
        r_h(x3, a2) = 1\\
    \end{cases}
\end{equation*}

Recall that our algorithm makes two important changes to the $Q$-learning with UCB-Hoeffding algorithm from \citet{jin2018q}. Namely, the initialization is changed according to equation (\ref{Initialization equations}) 
and the $Q$-function is only updated if the current state, action, and step is $(x1, a1, h=1)$. 
We refer to these two changes as assumption 1 (A1) and assumption 2 (A2) respectively.
We train an agent using three different algorithms:
\begin{enumerate}
    \item $Q$-learning with UCB-Hoeffding \cite{jin2018q}
    \item $Q$-learning with UCB-Hoeffding and Max-Optimal Initialization without A2
    \item $Q$-learning with UCB-Hoeffding and Max-Optimal Initialization (Algorithm \ref{algorithm 1})
\end{enumerate}
From the above enumeration, points 1 and 3 have already been discussed in detail and are shown explicitly as Algorithm \ref{algorithm 1} and Algorithm 1 in \citet{jin2018q} respectively. 
The goal of point 2 is to show that A2 is crucial for our theoretical regret bound (as seen in Theorem \ref{Hoeffding MO}). 
Since A2 is removed in point 2, we expect the total regret to be greater than that of our algorithm.
An important detail is that the $Q$-learning with UCB-Hoeffding and Max-Optimal Initialization without A2 algorithm is the same as our algorithm except with A2 removed, but also with $N_h(x, a)$ initialized to $0$ for all $(x, a, h) \in \mathcal{S} \times \mathcal{A} \times [H]$. Since now the $Q$-function is updated for all states, actions, and steps,  a visit count $N_h(x, a)$ must be kept for all $(x, a, h)$.

An agent was trained using each algorithm for $500$ episodes, $K = 500$. 
We set the probability term $p = 0.05$, corresponding to a probability of at least $1-p = 0.95$ of obtaining a total regret of $O(\sqrt{H^2T\iota'})$ as mentioned in Theorem \ref{Hoeffding MO}. 
The constant used for the bonus $b_t$ was set to $c=0.1$.
For the $Q$-learning with UCB-Hoeffding and Max-Optimal initialization algorithm we set $(x1, a1, h=1)$ as the non-optimal $Q$-value.

Recall that going left ($a1$) in the left-most state $x1$ causes the agent to hit a wall and remain in state $x1$. 
Action $a1$ is not the optimal action in state $x1$, since the agent will only receive a reward of $1$ for transitioning into the right-most state, which going left, $a1$ does not help achieve. 
The optimal action is for the agent to go right ($a2$) from state $x1$. 
Intuitively, the above initialization causes the agent to have low confidence in $(x1, a1, h=1)$ and therefore the agent will explore $(x1, a1, h=1)$ until it is confident enough that $(x1, a1)$ actually provides it with less total reward than $(x1, a2)$ at step $h=1$. 

As a way to measure performance per episode we introduce: 
$$\text{Per Episode Regret} = \text{PER}(k) = V_1^*(x_1^k) -V_1^{\pi_h^k}(x_1^k)$$ 
When the PER is summed over all episodes it gives the total regret. In our experiments, the PER was averaged over $50$ independent runs for each algorithm.


\section{Experimental Results}

\begin{figure}[t]
    \centering
    \includegraphics[width=0.5\textwidth]{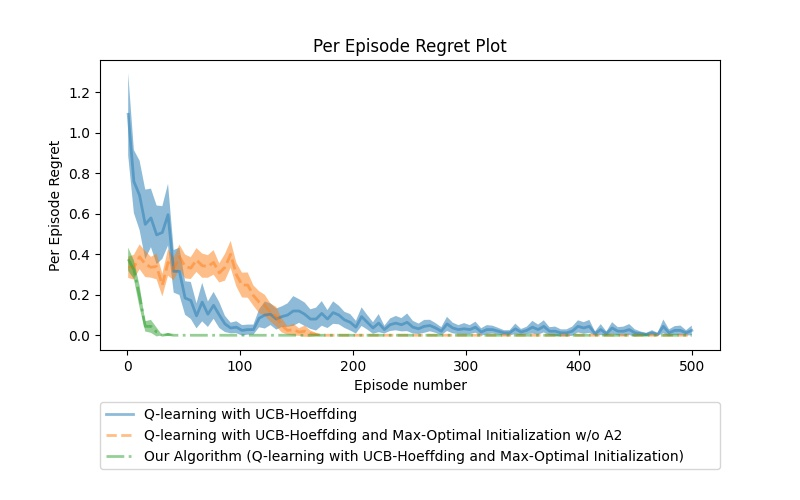}
    \caption{A Per Episode Regret plot of algorithm UCB-H, UCB-H MO, and UCB-H MO A2 averaged over $50$ independent runs. The shaded regions represent a 95\% confidence interval.} 
    \label{fig:PER1}
\end{figure}

Figure \ref{fig:PER1} shows the results of training an agent in the setting mentioned in the previous section. 
For convenience, we will refer to the $Q$-learning with UCB-Hoeffding and Max-Optimal Initialization algorithm as UCB-H MO, and the $Q$-learning with UCB-Hoeffding and Max-Optimal Initialization without A2 algorithm as UCB-H MO A2, and to the $Q$-learning with UCB-Hoeffding as UCB-H. 
In Figure \ref{fig:PER1} we observe that the per episode regret converges to zero fastest for UCB-H MO. 
A faster convergence to zero implies that the total regret of UCB-H MO is lower than that of UCB-H MO A2 and UCB-H, since the total regret is just the sum of the PER over all episodes. 
Recall from the theoretical results that the total regret of UCB-H MO is $O(\sqrt{H^2T\iota'})$, while UCB-H has a total regret of $O(\sqrt{H^4SAT\iota})$. 
Therefore, this is the expected behaviour of UCB-H MO when compared to UCB-H based on our theoretical results. 
Since UCB-H MO also converges faster than UCB-H MO A2 it supports our earlier claim that A2 is crucial to the theoretical regret bound we obtain (see Theorem \ref{Hoeffding MO}).


\section{Discussion and Future Work}

In summary, we showed that the total regret of $Q$-learning with UCB-Hoeffding and Max-Optimal Initialization is upper bounded by $O(\sqrt{H^2T\iota'})$. This total regret bound is tighter than that of the $Q$-learning with UCB-Hoeffding algorithm, $O(\sqrt{H^4SAT\iota})$ \cite{jin2018q}. This result provides theoretical justification for applying $Q$-function reuse on the $Q$-function with UCB-Hoeffding algorithm. Although we make the strong assumption that all but one of the Q-values are optimal, we believe that this provides a solid starting point for future work to build upon. 

We believe some interesting future directions are:
\begin{enumerate}
    \item Increasing the number of non-optimal $Q$-values. In this work it was assumed that the $Q$-function was optimal for all but one state, action, and step $(x1, a1, h=1)$. One possible next step would be to assume that the $Q$-values of two or more states, actions, and steps are non-optimal. This might make the analysis more complex because the proof of Lemma \ref{Diff in Q} assumed that $V_h^k(x_h^k) = V_h^*(x_h^k)$ for $h \geq 2, \forall k \in [K]$ which would no longer be true if one of the non-optimal $Q$-values occurred for $h \geq 2$.  
    \item Relaxing the $Q$-function optimality assumption for some $(x, a, h) \in \mathcal{S} \times \mathcal{A} \times [H]\backslash (x1, a1, h=1)$ to some notion of near-optimality. This work assumes that the $Q$-functions was optimal for all but one state, action, and step $(x1, a1, h=1)$. This assumption is generally unrealistic for a transferred $Q$-function because an agent rarely learns the optimal $Q$-function in an environment (instead, it often only reaches a near-optimal one). Such cases can potentially be modelled by assuming some sub-optimality of the $Q$-function initialization. For instance, for some subset of states, actions, and steps it could be assumed that $Q_h(x, a) \ge Q^*_h(x, a) - \epsilon, \ \forall (x, a, h), \epsilon \ge 0$
    \item Allowing the $Q$-function to be updated for some $(x, a, h) \in \mathcal{S} \times \mathcal{A} \times [H]\backslash (x1, a1, h=1)$. In this work it was assumed that we had prior knowledge of which states, actions, and steps the $Q$-function was optimally initialized (i.e. $Q_h(x, a) = Q^*_h(x, a), \ \forall \mathcal{S} \times \mathcal{A} \times [H] \backslash (x1, a1, h=1)$) and therefore we were able to explicitly choose to not perform $Q$-function updates at those states, actions and steps. It may not always be the case that this information is known. As such, the algorithm would not update the $Q$-function for all states, actions and steps. Modifications to the analysis done in this work would have to be made to provide an upper bound on the regret of such an algorithm because a $Q$-function update performed at an optimal state, action, and step might depend on the $Q$-function at a non-optimal state, action, and step. Such an update can potentially change the value of the $Q$-function at the optimal state, action, and step, causing for an increase in the number of non-optimal $Q$-values.  
\end{enumerate}



\begin{acks}
This work has taken place in the Intelligent Robot Learning (IRL) Lab at the University of Alberta, which is supported in part by research grants from the Alberta Machine Intelligence Institute (Amii) and NSERC, as well as a Canada CIFAR AI Chair.
\end{acks}



\newpage
\bibliographystyle{ACM-Reference-Format} 
\bibliography{ALA}


\end{document}